\documentclass{article}
\usepackage{ijcai20}

\usepackage{times}
\usepackage{soul}
\usepackage{url}
\usepackage[hidelinks]{hyperref}
\usepackage[utf8]{inputenc}
\usepackage[small]{caption}
\usepackage{graphicx}
\usepackage{amsmath}
\usepackage{amsthm}
\usepackage{nameref}
\usepackage{subcaption}
\usepackage{booktabs}
\usepackage{algorithm}
\usepackage{algorithmic}
\newtheorem{prop}{Proposition}

\title{Streaming Active Deep Forest for Evolving Data Stream Classification}

\iftrue
\author{
	Anh Vu Luong$^1$
	\and
	Tien Thanh Nguyen$^2$\And
	Alan Wee-Chung Liew$^{1}$
	\affiliations
	$^1$School of Information and Communication Technology, Griffith University, Australia\\
	$^2$School of Computing Science and Digital Media, Robert Gordon University, Aberdeen, Scotland, UK\\
	\emails
	vu.luong@griffithuni.edu.au
}
\fi

\begin{document}

\maketitle

\begin{abstract}
  In recent years, Deep Neural Networks (DNNs) have gained progressive momentum in many areas of machine learning. The layer-by-layer process of DNNs has inspired the development of many deep models, including deep ensembles. The most notable deep ensemble-based model is Deep Forest, which can achieve highly competitive performance while having much fewer hyper-parameters comparing to DNNs. In spite of its huge success in the batch learning setting, no effort has been made to adapt Deep Forest to the context of evolving data streams. In this work, we introduce the Streaming Deep Forest (SDF) algorithm, a high-performance deep ensemble method specially adapted to stream classification. We also present the Augmented Variable Uncertainty (AVU) active learning strategy to reduce the labeling cost in the streaming context. We compare the proposed methods to state-of-the-art streaming algorithms in a wide range of datasets. The results show that by following the AVU active learning strategy, SDF with only 70\% of labeling budget significantly outperforms other methods trained with all instances.
\end{abstract}

\section{Introduction}

Recent years have witnessed a remarkable success of Deep Neural Networks (DNNs) \cite{lecun2015deep} in various domains, including images, videos, audios, and text processing tasks. Though DNNs are extremely powerful, they have some limitations: (1) they require a huge amount of labeled data to achieve high performance; (2) training them is hard and slow with an enormous number of parameters; (3) their effectiveness highly depends on careful hyper-parameters tuning for different tasks. These problems are even more severe when applying DNNs to the online setting, where the model cannot reaccess historical data, leading to its slow convergence.

The success of DNNs is commonly attributed to its representation learning capability, which mainly relies on layer-by-layer processing of the feature information. This recognition inspired the emergence of deep ensemble methods, most notably the gcForest model \cite{zhou2017deep}, which can solve the above-mentioned problems of DNNs while keeping the representation learning ability and producing high prediction accuracy. In particular, gcForest has much fewer hyper-parameters in comparison to DNNs and can achieve better results across various domains when using the same setting.

Data stream mining has become increasingly important in recent years owing to the massive amount of real-time data generated by sensor networks, IoT devices, and system logs. Building a strong predictive model for data streams is, therefore, a crucial task for many applications. Unlike in traditional batch classification where we can store the entire dataset in memory and process them with unlimited time, here we consider the evolving data stream setting where the following learning paradigms and resource constraints need to be satisfied: (1) the model is ready to classify any sequentially arriving instances at any time; (2) we expect an infinite sequence of data processed under limited time and memory; (3) the data distribution may change over time (the appearance of concept drift \cite{webb2016characterizing}); (4) the model can only observe each instance once before discarding it.

In learning from evolving data streams, the labeling process may incur high costs and may require a great deal of human effort. Active learning studies how to wisely query the most informative instances instead of asking for all labels. An effective active learning strategy can save us a huge number of label requests while keeping the performance of the learner as high as possible. It also helps accelerate the learning process since the learner will be trained on fewer instances. 

In this work, we introduce a novel active classification method for evolving data streams. First, we present Streaming Deep Forest (SDF), which is an adaptation of the gcForest model for the stream setting. SDF retains the representation learning ability of gcForest by reusing its cascade structure. To update the model on the fly, we replace the classic Random Forest \cite{breiman2001random} at each layer by Adaptive Random Forest (ARF) \cite{gomes2017adaptive}, a high-performance forest model for the stream setting. Concerning the problem of concept drift, SDF incorporates an active drift detection strategy. More details of SDF is described in Section \ref{proposed_method}. Second, we enhance the Variable Uncertainty (VU) strategy \cite{vzliobaite2011active} to obtain a novel active learning method, namely Augmented Variable Uncertainty (AVU). We provide a theoretical proof that the VU strategy does not take the full advantage of the given budget, and then we propose AVU strategy to tackle this issue. Our contributions in this work are summarized as follows:

1) Streaming Deep Forest (SDF): We introduce a deep ensemble method, namely SDF, that achieves high prediction accuracy by exploiting the layer-by-layer processing of raw features. To the best of our knowledge, SDF is the first deep ensemble model being used under the data stream setting.

2) Augmented Variable Uncertainty (AVU) active learning strategy: We theoretically show a problem of the Variable Uncertainty (VU) strategy that it does not make full use of the given budget, and propose the AVU strategy to fix that issue.

3) Empirical analysis: We compare the proposed methods with a number of state-of-the-art algorithms for streaming context concerning a wide range of datasets.  The experiment results show that by following the AVU strategy, SDF significantly outperforms all the benchmark algorithms even when it uses only 70\% of the labeling budget.

In the next sections, we will discuss the background and related work (Section \ref{bg_rw}), followed by the proposed methods (Section \ref{proposed_method}) and experiments (Section \ref{exp}). Finally, Section \ref{concl} concludes this work and presents directions for future works.

\section{Background and Related Work} \label{bg_rw}
\subsection{Ensemble Methods and Deep Ensemble Methods for batch learning}

A multitude of ensemble systems are widely used in the traditional batch learning setting, including Bagging \cite{breiman1996bagging}, Boosting \cite{freund1997decision}, Random Subspace \cite{barandiaran1998random}, and Random Forest \cite{breiman2001random}. These methods are different in how they generate diversity in the ensemble. Bagging, for example, trains base learners on different bootstrap replicates obtained by using sampling with replacement of the training set. Meanwhile, Random Subspace pays attention to the feature space by training each base learner on a randomly selected subset of features. Random Forest extends Bagging by using Decision Trees as its base learners and choosing a random subset of features to be used for splits in each base tree.

Recently, the first ensemble-based deep model has been introduced, namely gcForest \cite{zhou2017deep}. It was constructed using multiple layers, each of which contains two Completely-Random Tree Forests and two Random Forests \cite{breiman2001random}. In detail, each forest in a layer outputs a class vector obtained by averaging the class distribution vectors of all the base decision trees. Then a concatenation of the original feature vector and four class vectors returned by four random forests is used as the input data for the next layer. The gcForest model achieves superior performance on a wide range of domains in comparison to DNNs and other ensemble algorithms. More importantly, gcForest has much fewer hyper-parameters than DNNs and performs robustly on various datasets by using the same parameter setting.

\subsection{State-of-the-art methods for evolving data streams} \label{sub22}

There are a massive number of methods for data stream classification. Here we only consider state-of-the-art algorithms according to their prediction performance and flexibility.

Almost all the strongest models for evolving data streams are ensemble-based methods, because they can handle concept drifts effectively by selectively removing or adding base learners when changes happen. Online Bagging \cite{oza2005online} is an adapted replicate of the classical Bagging algorithm, in which the Poisson(1) distribution is employed to simulate the behavior of bootstrap technique in an online manner. Leveraging Bagging \cite{bifet2010leveraging} enhances Online Bagging by adding more randomization to the input and output of the base learners and employing the ADaptive WINdow (ADWIN) drift detection algorithm \cite{bifet2007learning} to selectively reset the base models whenever concept drift occurs. Chen et al. proposed an online version of Smooth Boost \cite{servedio2003smooth}, namely Online Smooth Boost (OSBoost) \cite{chen2012online}, which aims to generate only smooth distributions that do not assign too much weight to a single instance. It is theoretically guaranteed that OSBoost can achieve arbitrarily small error rate as long as the number of weak learners and instances are sufficiently large. Adaptive Random Forest (ARF) \cite{gomes2017adaptive} aims to adapt the classical Random Forest to the data stream setting by employing the online bootstrap resampling, similar to Leveraging Bagging. To deal with concept drift, ARF uses two change detectors per base tree to detect warnings and drifts. In particular, when a warning is triggered, a background tree is created and updated without affecting the ensemble predictions. If the warning escalates to a drift after a period of time, the background tree replaces the corresponding base tree in the ensemble. Recently, Gomes et al. introduced Streaming Random Patches (SRP) \cite{gomes2019}, which resembles the classic Random Patches \cite{louppe2012ensembles} by combining the Random Subspace method and Online Bagging \cite{oza2005online}. SRP exploits the global subspace randomization (as in Random Subspace), while ARF takes advantage of local subspace randomization (as in Random Forest). In SRP, the drift detection and recovery strategy follows the procedure used in ARF.

An Online Deep Learning (ODL) framework has been proposed recently \cite{sahoo2017online}, which employs Hedge Backpropagation to overcome the slow convergence issue of DNNs in the online setting. However, ODL has no explicit mechanism to deal with changes in the data distribution, resulting in poor performance when concept drift occurs.


\subsection{Active Learning with evolving data streams} \label{bg_al}

In active learning setting for streaming data, the decision to request the true label for a data point must be made immediately when that instance arrives. Only a few active learning strategies have been proposed for evolving data streams. In \cite{vzliobaite2011active}, the first theoretically supported active learning framework for instance-incremental streaming data was introduced. The authors also proposed two novel active learning strategies, namely Variable Uncertainty (VU) and Variable Randomized Uncertainty (VRU), that can handle concept drift explicitly. The VU strategy employs a variable threshold, which adjusts itself based on arriving data points to align with the given budget. Meanwhile, in the VRU strategy, the labeling threshold is multiplied by a random variable that follows the normal distribution $ \mathcal{N}(0,1) $. This strategy labels the data points that are close to the decision boundary more often, but occasionally requests labels for some distant instances. Cesa-Bianchi et al. developed an online active learning method, namely Selective Sampling (SS) \cite{cesa2006worst}, using a variable labeling threshold $ b/(b + |p|) $, where $ b $ is a parameter, and $ p $ is the prediction of the perceptron. This method could be adapted to changes, although the authors did not explicitly handle concept drift. Xu et al. employed a Paired  Ensemble Framework to perform active learning from evolving data streams \cite{xu2016active}. In detail, an ensemble of two base learners is used to predict new instances and detect changes over time. Meanwhile, two active learning strategies (Random strategy and Variable Uncertainty strategy) work alternatively to look for the most informative instances.

\begin{figure}[t!] 
	\centering
	\includegraphics[width=\columnwidth]{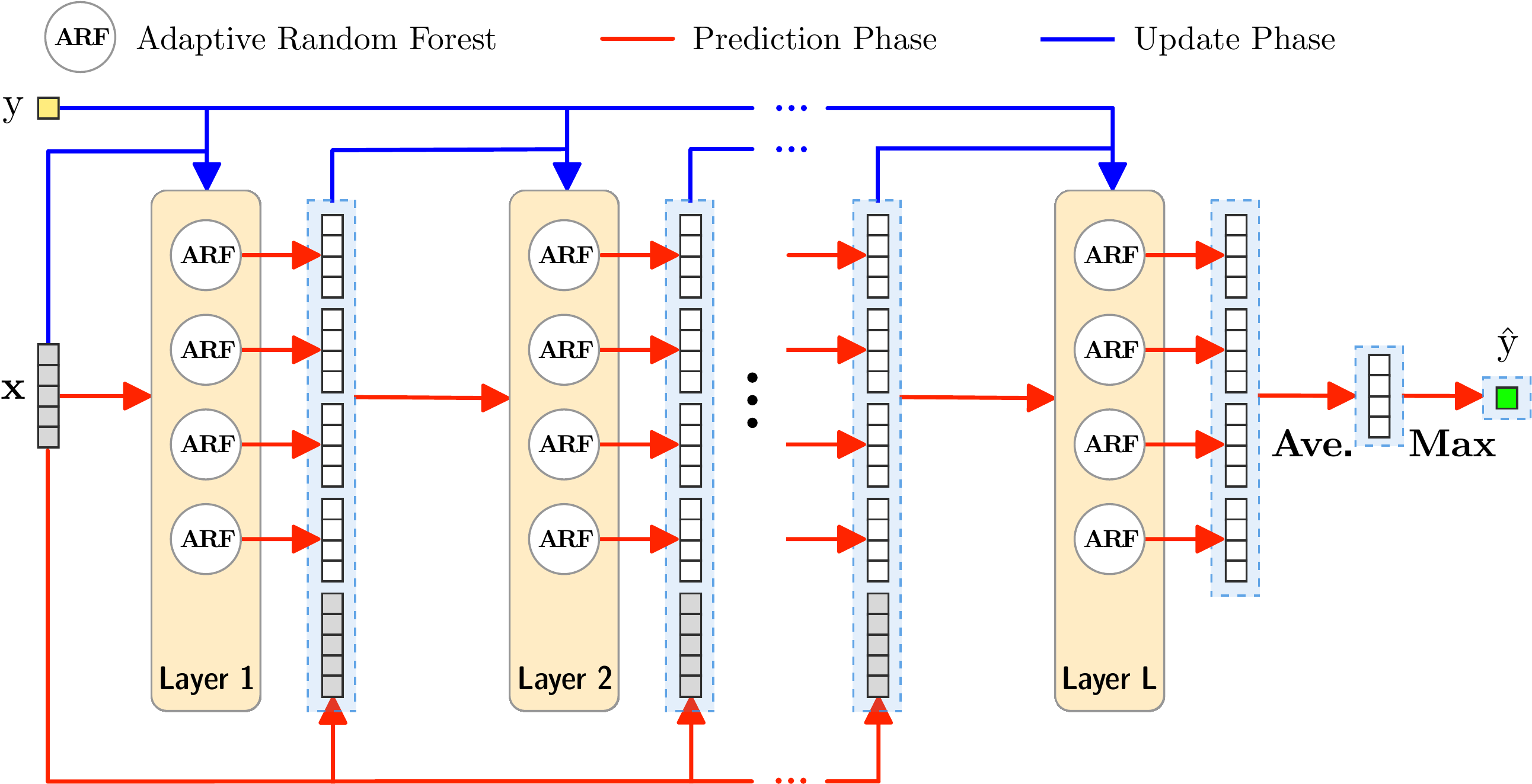}
	\caption{Streaming Deep Forest}
	\label{h1}
\end{figure}

\section{Proposed Methods} \label{proposed_method}
\subsection{Problem setting}

Let a data stream $ X=\{x_1,x_2,…,x_\infty\} $ be an infinite sequence of data points where $ x_k $ is a $ d $-dimensional vector of features. Correspondingly, let $ Y=\{y_1,y_2,…,y_\infty\} , y_k \in \{l_1, l_2,…, l_M\} $ be the sequence of class labels, such that an entry $ y_k $ in $ Y $ is the true label of $ x_k $ in $ X $. Most of the existing works on data stream classification assume that the true label $ y_k $ is available before the next data point $ x_{k+1} $ arrives.

Furthermore, we assume evolving data streams, in which concept drifts may occur over time. The appearance of concept drifts influences the decision boundary and damages the current learned model. Here, an i.i.d. assumption is made for each concept, i.e., each concept is treated as a separate i.i.d. stream. As a result, we have to deal with a series of i.i.d. streams.


\subsection{Streaming Deep Forest}

It is widely acknowledged that the success of deep neural networks is attributed to its representation learning ability, which mostly relies on layer-by-layer processing of the features. Similarly, gcForest \cite{zhou2017deep} generates a deep forest ensemble based on cascade structure to perform representation learning. Specifically, each layer of gcForest takes the output of its previous layer as the feature information and transmits its processing result to the next layer.

Under the data stream setting, we employ the cascade structure of gcForest to retain the representation learning capability. However, we change the constituents and propose a new training scheme to make the model able to learn incrementally from data streams. Figure \ref{h1} illustrates the proposed method, which we refer to as Streaming Deep Forest (SDF).

Each layer is represented by an ensemble of Adaptive Random Forests (ARF), i.e., an ensemble of ensembles. Due to the fact that ARF is an online classifier, we can update all layers on the fly and use them to make predictions at any time. In addition, to promote diversity, a crucial factor in ensemble learning \cite{zhou2012ensemble}, we construct each layer using four ARFs with different hyper-parameters.

Consider an arbitrary data point, each ARF will output an estimate of the posterior distribution, which is the weighted average across all trees’ class distribution. In more detail, ARF employs the Hoeffding Tree \cite{domingos2000mining} algorithm with Naïve Bayes classifier at the leaves as the base learner, which we call Hoeffding Naïve Bayes Tree (HNBT). Note that in ARF, each base tree limits its splits to $ m (m < d) $ randomly selected features.

The posterior distribution given by each ARF forms a class vector. We then concatenate all these class vectors and the original feature vector to input to the next layer. Lets take a problem that aims to classify 5D feature vectors into four classes as an example. In this case, each of four ARFs outputs a four-dimensional class vector; thus, the next layer will receive $ 21(=4\times4+5) $ features.

When the true label $ y_k $ of the data point $ x_k $ is revealed, we update each layer by using $ y_k $ and the input vector (the concatenation of the previous layer’s outputs and $ x_k $). In each layer, the update process can be easily parallelized since the four ARFs are independently executed, and so are the base trees of each ARF. In this work, we use the CPU multi-processor architecture to parallelize each layer of SDF. 

Regarding the issue of concept drift, SDF follows the active change detection and recovery strategy used in ARF, which is described in Sub-section \ref{sub22}


\subsection{Augmented Variable Uncertainty Strategy}

We study active learning for instance-incremental streaming data, where concept drift is expected to occur. The true label can be requested immediately or never, as the data points are discarded from memory after being used. The goal is to maximize the prediction accuracy over time while keeping the labeling cost fixed within an allocated budget.

Given a data stream $ X=\{x_1,x_2,…,x_\infty\} $, we assume that the labeling cost is the same for any data point. A budget $ B $ is imposed to request the true labels, i.e., the maximum fraction of the incoming data points that we can obtain the true labels. If $ B=1 $, for example, all arriving data points are labeled, whereas if $ B=0.6 $, we can request the true labels of up to 60\% of the arriving data points.

\begin{algorithm}[tb]
	\caption{Active Learning with the AVU strategy}
	\label{algo1}
	\textbf{Input}: $ x_k $ - incoming instance, $ B $ - budget, $ s $ - adjusting step \\
	\textbf{Output}: $ label \in \{\textbf{true}, \textbf{false}\} $ specifies whether to query the true label $ y_k $
	\begin{algorithmic}[1] 
		\STATE Initialize labeling cost $ c = 0 $, labeling threshold $ \theta = 1 $
		\IF [\textit{budget is not exceeded}]{$ (c/ k < B) $} 
		\STATE $ \hat{p}=\max_{y} P(y|x_k), y \in \{l_1, l_2,...,l_M\} $
		\IF [\textit{certainty below the threshold}]{$ \hat{p} < \theta $}
		\STATE $ c = c + 1 $; $ \theta = \theta(1 - s) $
		\RETURN \textbf{true}
		\ELSE [\textit{certainty is good}]
		\STATE $ \theta = \theta(1 + s) $
		\STATE Generate a uniform random variable $ \rho \in [0,1] $
		\RETURN $ \rho < 2\times (B - 0.5) $
		\ENDIF
		\ELSE [\textit{budget is exceeded}]
		\RETURN \textbf{false}
		\ENDIF
	\end{algorithmic}
\end{algorithm}

In this work, we improve the Variable Uncertainty (VU) strategy \cite{vzliobaite2011active}. Here, the certainty is measured by using the posterior probability estimates, i.e., the higher the maximum of the posterior probabilities is, the more certain the prediction is. This strategy tries to label the least certain data points within a time period by using a variable certainty threshold, which adapts itself according to the arriving instances. Specifically, VU queries labels for the instances with their certainty scores below the variable threshold. In stable data concept (no change happens), the classifier becomes more confident about its predictions; thus, the certainty threshold will grow to cover some high-certainty data points. By contrast, if a concept drift occurs and lots of labeling requests suddenly appear, then the certainty threshold is contracted to be able to query labels for the most uncertain data points first.

A problem with the VU strategy is that it does not take full advantage of the given budget $ B $. In Proposition \ref{prop1}, we show that by following this strategy, we only spend a maximum budget of 0.5 in expectation. As a consequence, when the budget $ B>0.5 $, it will miss out a fraction of about $ (B-0.5) $ of incoming instances that we can ask for their labels. To address this issue, we proposed the Augmented Variable Uncertainty (AVU) strategy in Algo \ref{algo1}. The difference between AVU and VU is that when ``the certainty is good", VU always refuses to query labels, whereas AVU requests labels with a probability $ P = 2\times (B - 0.5) $. This allows AVU to take the full advantage of the labeling budget $ B $.


Requesting labels when “certainty is good” is beneficial in evolving data streams, as changes can happen everywhere in the instance space. Thus, if we refuse to query labels for certain data points, some regions will never be observed, and we never know that concept drifts are occurring in those regions and, therefore, never adapt.

\begin{table}[]
	\caption{Datasets used in the experiments}
	\label{table_data}
	\resizebox{\columnwidth}{!}{%
		\begin{tabular}{lrcccc}
			\hline
			\textbf{Dataset}          & \textbf{\# Instances} & \textbf{\# Classes} & \textbf{\# Features} & \textbf{Type}      & \textbf{Drifts} \\ \hline
			Airlines         & 539,383                                                   & 2                                                         & 7                                                       & Real      & -      \\
			Covtype          & 581,012                                                   & 7                                                         & 54                                                      & Real      & -      \\
			Adult            & 48,842                                                    & 2                                                         & 14                                                      & Real      & -      \\
			Electricity      & 45,312                                                    & 2                                                         & 8                                                       & Real      & -      \\
			KDDCup99         & 4,898,431                                                 & 23                                                        & 41                                                      & Real      & -      \\
			Mnist\_a         & 70,000                                                    & 10                                                        & 784                                                     & Real      & A      \\
			Nomao            & 34,465                                                    & 2                                                         & 118                                                     & Real      & -      \\
			Vehicle          & 98,528                                                    & 2                                                         & 100                                                     & Real      & -      \\
			20\_newsgroups   & 399,940                                                   & 2                                                         & 1000                                                    & Real      & -      \\
			AGR\_a           & 1,000,000                                                 & 2                                                         & 9                                                       & Synthetic & A      \\
			AGR\_g           & 1,000,000                                                 & 2                                                         & 9                                                       & Synthetic & G      \\
			BNG\_tic-tac-toe & 39,366                                                    & 2                                                         & 9                                                       & Synthetic & N      \\
			BNG\_vote        & 131,072                                                   & 2                                                         & 16                                                      & Synthetic & N      \\
			BNG\_segment     & 1,000,000                                                 & 7                                                         & 19                                                      & Synthetic & N      \\
			HYPER            & 1,000,000                                                 & 2                                                         & 10                                                      & Synthetic & F      \\
			RBF\_f           & 1,000,000                                                 & 5                                                         & 10                                                      & Synthetic & F      \\
			RBF\_m           & 1,000,000                                                 & 5                                                         & 10                                                      & Synthetic & M      \\
			RTG              & 1,000,000                                                 & 2                                                         & 10                                                      & Synthetic & N      \\
			SEA\_a           & 1,000,000                                                 & 2                                                         & 3                                                       & Synthetic & A      \\
			SEA\_g           & 1,000,000                                                 & 2                                                         & 3                                                       & Synthetic & G      \\ \hline
		\end{tabular}
		
	}
	\rule{0pt}{1.5ex}
	{\tiny $ {\ \ } $ (A) Abrupt, (G) Gradual, (M) Incremental (moderate), (F) Incremental (fast), and (N) No drift.}
\end{table}

\section{Experiments} \label{exp}

We compared the parallel implementation of SDF against state-of-the-art algorithms for evolving data streams, both concerning prediction accuracy and CPU run time. We used the test-then-train strategy, where each instance is first used for testing and then for training, to evaluate the accuracy of each classification method. The benchmark algorithms used in the comparison were the Online Deep Learning (ODL) framework, Leverage Bagging (LB), Online Smooth Boosting (OSB), Adaptive Random Forest (ARF), and Streaming Random Patches (SRP). These are recently proposed methods that consistently outperform other classifiers, as shown by experiments in the literature \cite{gomes2017adaptive,gomes2019,sahoo2017online}.

To evaluate the proposed active learning method AVU, we compared it to three techniques: Variable Uncertainty (VU), Variable Randomized Uncertainty (VRU), and Selective Sampling (SS). The ideas of these methods are briefly discussed in Sub-section \ref{bg_al}, and the implementations of them are available in the MOA library\footnote{\url{https://moa.cms.waikato.ac.nz}}.

Regarding the hyper-parameters, we used Hoeffding Tree (HT) as the base classifier for all ensemble-based methods. The number of layers of SDF was set to 3 when comparing to other benchmarks. Each layer contained 4 ARFs, each of which comprised of 50 base trees. We, therefore, used 200 HTs as the base learners for other ensemble algorithms. We employed ADWIN to be the drift detector for all ensemble methods that rely on active drift detection (i.e., SDF, ARF, SRP, LB). We also incorporated ADWIN to ODL to help it deal with concept drift. The confidence bound $ \delta $ of ADWIN was set to $ \delta=10^{-4} $ for warning detection and $ \delta=10^{-5} $ for drift detection in SDF, ARF, and SRP. In LB and ODL, $ \delta $ was set to its default value $ \delta=0.002 $. In the AVU, VU, and VRU active learning strategies, the adjusting step $ s $ was set to $ s=0.01 $ as used in \cite{vzliobaite2011active}. When comparing to other benchmark algorithms, we added a variant of SDF that follows AVU active learning strategy with budget $ B = 0.7 $, which we refer to as SDF(B=0.7). Other hyper-parameters that are not mentioned here were set to their default values, as shown in the original papers, and they can also be found in the MOA library. 

We conducted experiments on 20 datasets, including 11 synthetic data streams and 9 real-world datasets. These datasets have been extensively used in the data stream literature, containing concept drifts (gradual, abrupt, and incremental) and stationary streams. More details of these datasets are shown in Table \ref{table_data}.

\begin{table}[] 
	\centering
	\caption{Test-then-train accuracy(\%)}
	\label{table_accuracy}
	\resizebox{\columnwidth}{!}{%
		\begin{tabular}{lccccccc}
			\hline
			\multicolumn{1}{c}{\textbf{}} & \textbf{ARF} & \textbf{SRP}     & \textbf{LB} & \textbf{OSB} & \textbf{ODL}     & \textbf{SDF}     & \textbf{\begin{tabular}[c]{@{}c@{}}SDF\\ (B=0.7)\end{tabular}} \\ \hline
			Airlines                      & 66.4646      & 68.4972          & 63.7109     & 65.0028      & 61.3000          & 68.4934          & \textbf{68.5680}                                               \\
			Covtype                       & 92.5220      & 94.8010          & 93.5106     & 87.2538      & 89.5800          & \textbf{95.7142} & 95.6591                                                        \\
			Adult                         & 83.9503      & 84.6485          & 84.2554     & 83.4200      & 76.0700          & 84.6812          & \textbf{84.7426}                                               \\
			Electricity                   & 89.0228      & 89.4333          & 88.4225     & 88.3077      & 73.8900          & \textbf{91.3180} & 91.0465                                                        \\
			KDDCup99                      & 99.9716      & \textbf{99.9768} & 99.9503     & 99.8541      & 99.9600          & 99.9737          & 99.9719                                                        \\
			Mnist\_a                      & 91.5614      & 84.4300          & 60.9343     & 28.7500      & 88.8300          & 93.1129          & \textbf{93.5343}                                               \\
			Nomao                         & 97.0985      & 97.2813          & 95.8741     & 93.6980      & 96.2000          & \textbf{97.5337} & 97.5018                                                        \\
			Vehicle                       & 85.0652      & 84.6856          & 84.8957     & 78.1057      & 85.4000          & \textbf{86.8271} & 86.7388                                                        \\
			20\_newsgroups                & 99.6534      & 99.7010          & 99.4957     & 98.7401      & 99.5600          & 99.7082          & \textbf{99.7132}                                               \\
			AGR\_a                        & 90.7030      & 93.0238          & 89.8923     & 93.0665      & 60.8900          & 94.7449          & \textbf{94.8459}                                               \\
			AGR\_g                        & 87.0745      & 89.4430          & 86.7219     & 90.4978      & 60.0800          & 91.5984          & \textbf{91.7029}                                               \\
			BNG\_tic-tac-toe              & 78.4103      & 77.1884          & 77.8565     & 75.3340      & 70.8400          & 78.9615          & \textbf{78.9920}                                               \\
			BNG\_vote                     & 96.9841      & 96.8628          & 96.9643     & 96.6232      & 96.6100          & \textbf{97.1886} & 97.1649                                                        \\
			BNG\_segment                  & 87.3781      & 86.9372          & 87.1968     & 85.9996      & 86.3600          & 87.5945 & \textbf{87.6018}                                                        \\
			HYPER                         & 85.2711      & 84.9455          & 87.3113     & 89.1766      & \textbf{91.8000} & 88.4881          & 88.8437                                                        \\
			RBF\_f                        & 73.7903      & 75.2759          & 63.4204     & 43.4631      & 61.7800          & \textbf{78.8835} & 77.9373                                                        \\
			RBF\_m                        & 86.2290      & 85.1323          & 84.8242     & 66.8997      & 82.9800          & 87.8207          & \textbf{87.9004}                                               \\
			RTG                           & 94.0765      & 90.7654          & 97.8457     & 94.6737      & 82.6200          & \textbf{98.0519} & 97.9177                                                        \\
			SEA\_a                        & 89.6332      & 88.2105          & 86.9402     & 88.9516      & 86.6100          & 89.7021          & \textbf{89.7040}                                               \\
			SEA\_g                        & 88.9488      & 87.4435          & 88.4956     & 88.4597      & 85.6900          & 89.0484          & \textbf{89.0599}                                               \\ \hline
			\textbf{Avg Rank Real}        & 4.11         & 3.22             & 5.33        & 6.56         & 5.44             & \textbf{1.67}    & \textbf{1.67}                                                  \\
			\textbf{Avg Rank Synt.}        & 3.91         & 4.91             & 4.73        & 4.91         & 6.18             & 1.91             & \textbf{1.45}                                                  \\
			\textbf{Avg Rank}             & 4.00            & 4.15             & 5.00           & 5.65         & 5.85             & 1.80             & \textbf{1.55}                                                   \\ \hline
		\end{tabular}
	}
\end{table}

\begin{figure*}[t!] 
	\centering
	\includegraphics[width=2\columnwidth]{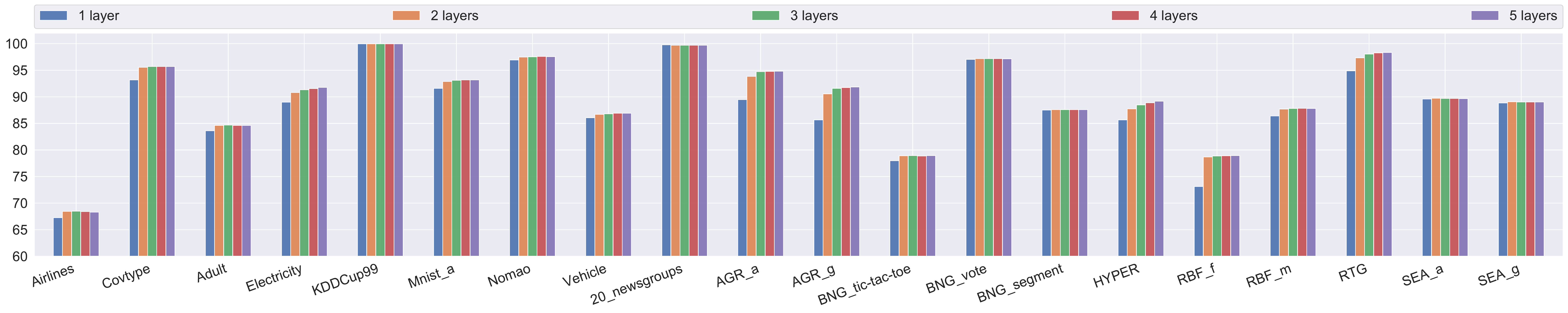}
	\caption{Test-then-train accuracy of SDF using different numbers of layers}
	\label{h3}
\end{figure*}

\subsection{Augmented Variable Uncertainty strategy}

First, we designed an experiment to confirm Proposition \ref{prop1} by examining the fraction of labeling requests of all active learning methods on the Electricity dataset when the labeling budget $ B>0.5 $. Figure \ref{h4} shows the results for $ B=0.7 $ and $ B=0.9 $. Clearly, in both cases, the labeling amount that VU and VRU request quickly converges to 0.5, which is consistent with our proposition. Note that we only considered VU in the proof, but it can be easily extended for VRU.

Figure \ref{h5} shows the comparisons of AVU against other active learning strategies on the Electricity and Airlines datasets given different values of the budget. When $ B\leq0.5 $, the proposed method yields the same result as the VU method, and the accuracy goes up when more budget is given. By contrast, in cases of $ B=0.7 $ and $ B=0.9 $, the performance of VU no longer increases, while the performance of AVU keeps rising. This observation demonstrates that it is beneficial to take full advantage of the given budget. In comparison to SS and VRU, AVU completely outperforms them in almost all cases. The only exception is the Electricity dataset with $ B=0.1 $, where SS yields higher accuracy than AVU.

\begin{figure}[t!] 
	\centering
	\includegraphics[width=\columnwidth]{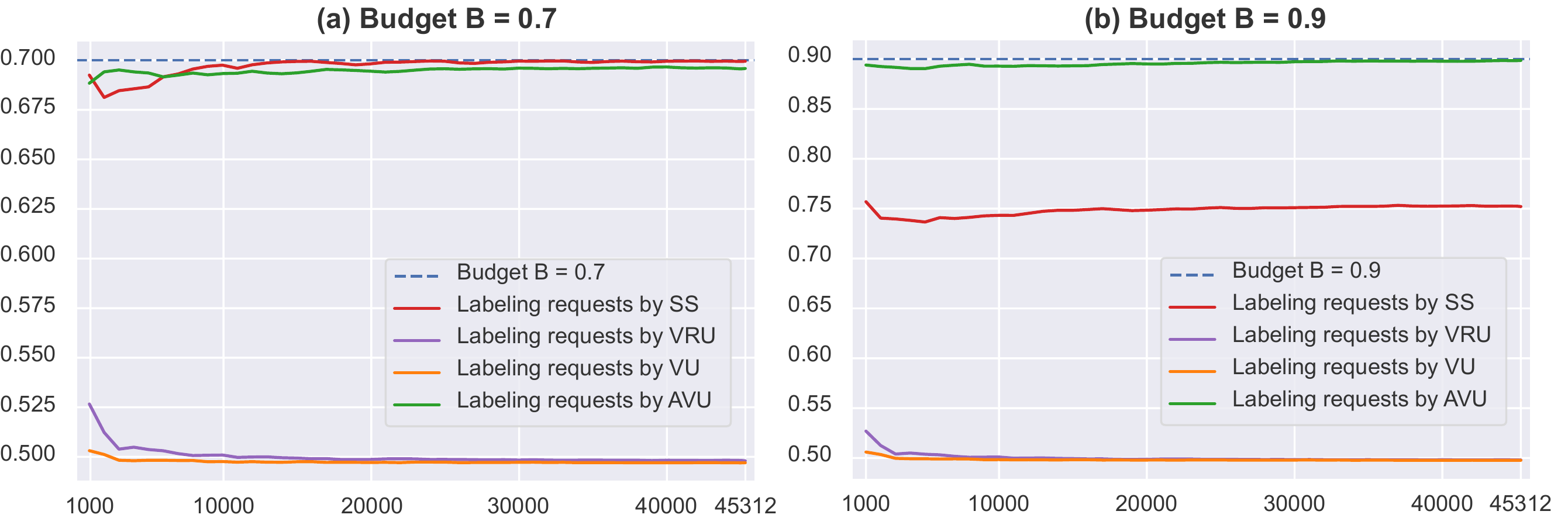}
	\caption{Electricity - Labeling costs over time}
	\label{h4}
\end{figure}

\begin{figure}[t!]
	\centering
	\includegraphics[width=\columnwidth]{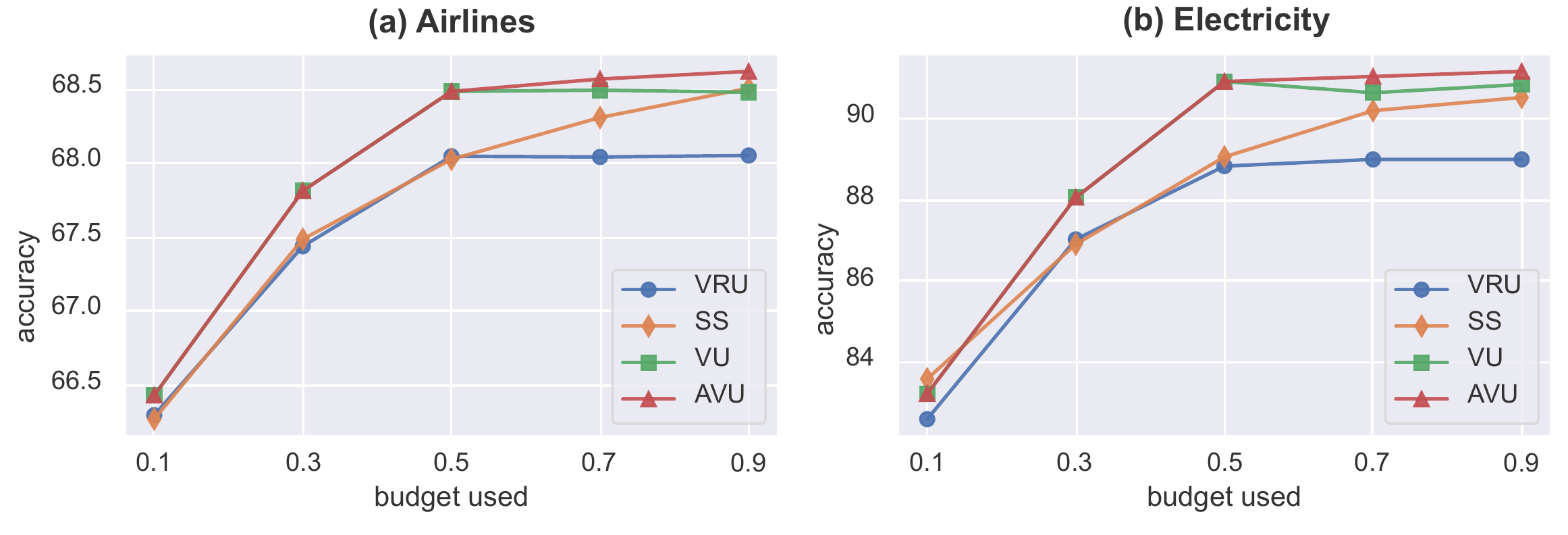}
	\caption{Accuracies given a budget. \textbf{a} Airlines. \textbf{b} Electricity}
	\label{h5}
\end{figure}

\subsection{Streaming Deep Forest vs. Others}
Table \ref{table_accuracy} shows the accuracy of SDF, SDF(B=0.7) and other algorithms on 20 datasets. Since some methods may perform better on synthetic data while not so well in general, we present both the average ranking for the real-world datasets (Avg Rank Real) and the average ranking for the synthetic datasets (Avg Rank Synt.) alongside the general average ranking for all datasets (Avg Rank). The result shows that SDF variants consistently rank first on almost all datasets (18/20) except for the KDDcup99 dataset and HYPER dataset, where they still yield reasonable performance. An interesting observation here is that SDF(B=0.7) achieves better average ranking than SDF though it queries only 70\% of the true labels for training. On real datasets, they both obtain the best average ranking (1.67), whereas SDF(B=0.7) performs slightly better than SDF on synthetic datasets.

To assess the statistical significance of the comparisons, we apply the Friedman test and the Nemenyi post-hoc test with the significance level $ \alpha=0.05 $ to evaluate multiple methods on multiple datasets \cite{demvsar2006statistical}. The Friedman test rejected the hypothesis that ``all methods perform equally". Figure \ref{nemenyi} illustrates the results of the post-hoc tests regarding the accuracy and the run time. In terms of accuracy, Figure \ref{nemenyi_a} shows that the proposed methods (SDF and SDF(B=0.7)) significantly outperform all the benchmark algorithms, while no significant difference has been found among ARF, SRP, LB, OSB, and ODL. Meanwhile, the run time of SDF is high in comparison to other methods, as shown in Figure \ref{nemenyi_b}, which is attributable to its multi-layer structure. Fortunately, by using active learning, SDF(B=0.7) performs much faster than SDF and obtains comparable run time to LB and SRP.

\begin{figure}[t!]
	\begin{subfigure}{.5\textwidth}
		\centering
		\includegraphics[width=.9\linewidth]{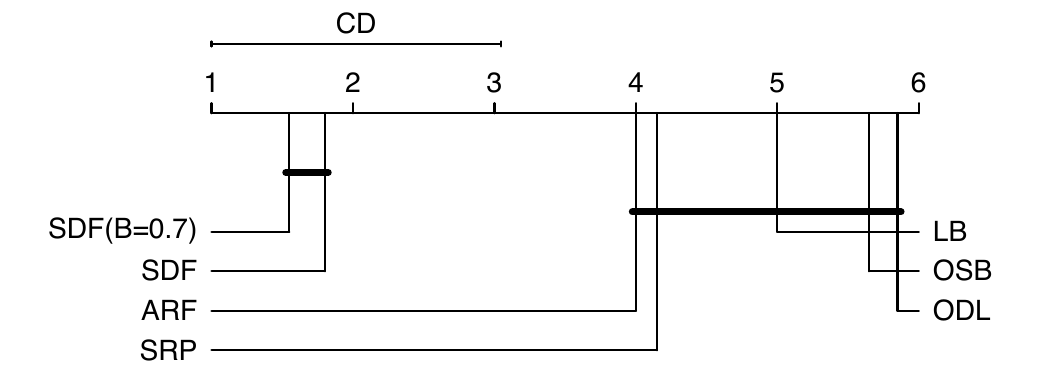}  
		\caption{}
		\label{nemenyi_a}
	\end{subfigure}
	\begin{subfigure}{.5\textwidth}
		\centering
		\includegraphics[width=.9\linewidth]{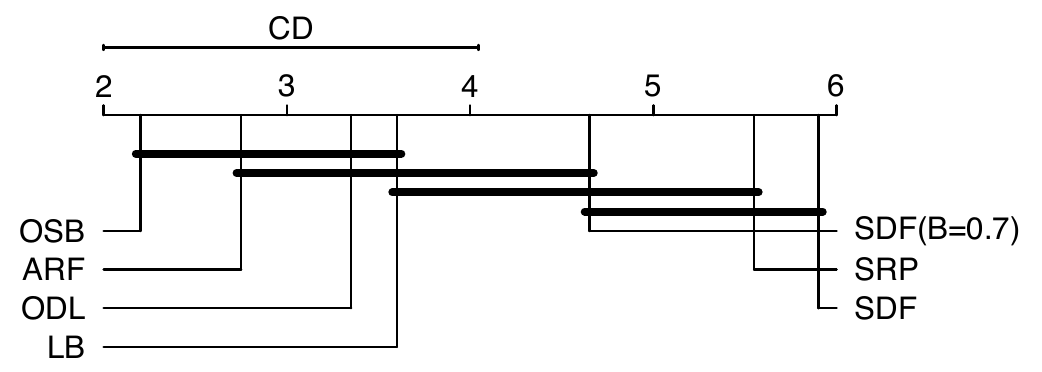}  
		\caption{}
		\label{nemenyi_b}
	\end{subfigure}
	\caption{Nemenyi test ($ \alpha = 0.05 $). \textbf{a} Accuracy. \textbf{b} Run time}
	\label{nemenyi}
\end{figure}

\subsection{Effect of hyper-parameters}

The only hyper-parameter of SDF apart from those of ARF base learner is the number of layers. Figure \ref{h3} shows the accuracy of SDF when we vary this hyper-parameter from 1 to 5. It is clear that SDF with only one layer performs much worse than that with more layers. In addition, there is an upward trend in almost all datasets, meaning that adding more layers tends to give better (or at least equal) accuracy. Therefore, we recommend using SDF with at least two layers when having low computing power; otherwise, use as many layers as the hardware can handle.

In the batch learning setting where the training time is unlimited, the classification model can be very deep with many layers. However, in the stream setting, the model is expected to process instances at least as fast as new instances are available. Thus, we only use up to five layers to align with our computation resources, but the proposed method can be directly extended to many more layers when more computing power is available.

\section{Conclusions} \label{concl}
In this work, we have adapted the gcForest model to the context of evolving data streams and proposed the Streaming Deep Forest. In particular, we exploited the cascade structure of gcForest to retain its representation learning ability, changed the base forest model at each layer to ARF, and employed an online training scheme to update SDF on the fly. We compared SDF to various state-of-the-art streaming classification methods over 20 datasets from both real-world applications and synthetic data generators. We also proposed an active online learning framework for evolving data streams, namely Augmented Variable Uncertainty. Our experiments showed that by following the AVU active learning strategy, SDF with only 70\% of the true labels significantly beats other benchmark methods trained with all the true labels.

In future work, we will study how to make SDF a deeper model while keeping the run time reasonable by considering sparse structures. When having a very deep model, an online weighted scheme for the layers can be employed to reduce the effort to tune the number of layers.

\vspace{14pt}
\begin{prop} \label{prop1}
	Given a data stream $ X=\{x_1,x_2,…,x_\infty\} $, a classifier $ L $, and a small positive number $ s $ (e.g. $ s=0.01 $). Let $ u_k $ be the certainty score of $ L $ on $ x_k $ which lies in the range $ [a,b]  (0\leq a < b) $, and $ \theta $ be a variable certainty threshold. Consider the following strategy:
	\begin{itemize}
		\item Initialize the certainty threshold $ \theta_1=b $
		\item For all $ k=1,2,... $
		\begin{itemize}
			\item If $ u_k<\theta_k $ then $ \theta_{k+1}=\theta_k (1-s) $
			\item If $ u_k\geq \theta_k $ then $ \theta_{k+1}=\theta_k (1+s) $
		\end{itemize}
	\end{itemize}
	Let $ \bar{\theta}_k $ be the expectation of the threshold at the $ k $-th instance. If we follow the above strategy, then the probability $ P(u_k<\bar{\theta}_k) $ converges to 0.5 when $ k $ approaches infinity.
\end{prop}


\begin{proof}
	Assume the certainty score $ u $ is uniformly distributed from $ a $ to $ b $: $ u \sim uniform(a,b) $, which means $ P(u_k < \bar{\theta}_k) = \frac{\bar{\theta}_k - a}{b-a}, \text{and } P(u_k \geq \bar{\theta}_k) = \frac{b - \bar{\theta}_k}{b - a} $. Hence, the expectation $ \bar{\theta}_{k + 1} $ is:
	\begin{multline}\label{theta k+1}
	\bar{\theta}_{k + 1} 
	= P(u_k < \bar{\theta}_k)\times \bar{\theta}_k(1 - s) + P(u_k \geq \bar{\theta}_k) \times \bar{\theta}_k(1 + s)
	\\= \frac{\bar{\theta}_k - a}{b - a}\times \bar{\theta}_k (1 - s) + \frac{b - \bar{\theta}_k}{b - a}\times \bar{\theta}_k (1 + s) \hspace{5px}
	\end{multline}
	To prove $ P(u_k < \bar{\theta}_k)  \xrightarrow[k\to\infty]{} 0.5$, we will prove that $   \bar{\theta}_k \xrightarrow[k\to\infty]{} \frac{a+b}{2}$.
	
	First, we use induction to show that $ \bar{\theta}_k \geq \frac{a + b}{2} $ for all $ k \geq 1 $. Assume that $ \bar{\theta}_k \geq \frac{a+b}{2} $. From \eqref{theta k+1}, the inequality $ \bar{\theta}_{k + 1} \geq \frac{a+b}{2} $ is equivalent to:
	\begin{equation}
	[2\bar{\theta}_k - (a+b)][2s\bar{\theta}_k - (b-a)] \leq 0
	\end{equation}
	which is satisfied due to the small value of $ s $ and the induction assumption. Hence, the inequality $ \bar{\theta}_{k + 1} \geq \frac{a+b}{2} $ is also satisfied. By induction, we have:
	\begin{equation} \label{bichan}
	\bar{\theta}_k \geq \frac{a + b}{2} \text{ for all } k \geq 1 
	\end{equation}
	
	Second, we show that $ \{ \bar{\theta}_k \} $ is a decreasing sequence, or $ \bar{\theta}_{k + 1} \leq \bar{\theta}_k $ for all $ k=1,2,... $. Substituting \eqref{theta k+1} to this inequality, we have:
	\begin{equation} 
	s\bar{\theta}_k[-2\bar{\theta}_k + (a + b)] \leq 0
	\end{equation}
	which holds due to \eqref{bichan}. Consequently, the inequality $ \bar{\theta}_{k + 1} \leq \bar{\theta}_k $ holds. Combining with \eqref{bichan}, we have $ \{\bar{\theta}_k\} $ is a decreasing and bounded below sequence. Therefore, it is converging. Now, let its limit be $ l = \lim_{k \to \infty} \bar{\theta}_k$. When $ k $ approaches $ \infty $, we have:
	\begin{equation*}
	l = \bar{\theta}_{k + 1}
	= \frac{(l-a)l(1-s)}{b-a} + \frac{(b - l)l(1+s)}{b-a}
	\end{equation*}
	which is equivalent to $ l = \frac{a + b}{2} $, or $ \lim_{k \to \infty} \bar{\theta}_k = \frac{a + b}{2} $. Therefore, $ P(u_k < \bar{\theta}_k) \xrightarrow[k \to \infty]{} 0.5$.
\end{proof}

\bibliographystyle{named}
\bibliography{sdf}


\end{document}